\documentclass{article}
\newcommand{\paperTitle}{Gap Safe Screening Rules for Fast Training\\ of Robust Support Vector Machines under Feature Noise}

\newcommand{\paperAbstract}{
    Robust Support Vector Machines (R-SVMs) address feature noise by adopting a worst-case robust formulation that explicitly incorporates uncertainty sets into training.
    While this robustness improves reliability, it also leads to increased computational cost.
    In this work, we develop safe sample screening rules for R-SVMs that reduce the training complexity without affecting the optimal solution.
    To the best of our knowledge, this is the first study to apply safe screening techniques to worst-case robust models in supervised machine learning.
    Our approach safely identifies training samples whose uncertainty sets are guaranteed to lie entirely on either side of the margin hyperplane, thereby reducing the problem size and accelerating optimization.
    Owing to the nonstandard structure of R-SVMs, the proposed screening rules are derived from the Lagrangian duality rather than the Fenchel–Rockafellar duality commonly used in recent methods.
    Based on this analysis, we first establish an ideal screening rule, and then derive a practical rule by adapting GAP-based safe regions to the robust setting.
    Experiments demonstrate that the proposed method significantly reduces training time while preserving classification accuracy.
}

\newcommand{\paperDate}{20/09/2025}

\newcommand{\authorLe}{Thu-Le Tran}
\newcommand{\affilLe}{Can Tho University, Vietnam}
\newcommand{\emailLe}{ttle@ctu.edu.vn}

\newcommand{\authorKien}{Kien Trung Nguyen}
\newcommand{\affilKien}{Can Tho University, Vietnam}
\newcommand{\emailKien}{trungkien@ctu.edu.vn}

\newcommand{\authorHau}{Tan-Hau Nguyen}
\newcommand{\affilHau}{Can Tho University, Vietnam}
\newcommand{\emailHau}{tanhau7420@gmail.com}

\usepackage{xcolor}

\usepackage{authblk}

\usepackage[left=3cm, right=3cm]{geometry}

\usepackage{amsthm}

\usepackage[square,authoryear]{natbib}

\usepackage[colorlinks=true,
            linkcolor=blue,
            citecolor=red,
            urlcolor=blue]{hyperref}

\usepackage{algorithm} \usepackage{algpseudocode}

\usepackage{float}

\usepackage{amsmath}
\usepackage{amssymb}

\usepackage{booktabs}

\usepackage{pgfplots}
\pgfplotsset{compat=1.18}
\usetikzlibrary{calc}

\usepackage{xcolor}

\theoremstyle{plain}
\newtheorem{theorem}{Theorem}
\newtheorem{lemma}{Lemma}

\theoremstyle{definition}

\theoremstyle{remark}
\newtheorem{remark}{Remark}

\title{\paperTitle}
\date{\paperDate}
\author[1]{\authorHau}
\author[2*]{\authorLe}
\author[3]{\authorKien}
\affil[1]{\affilHau \\ \texttt{\emailHau}}
\affil[2]{\affilLe \\ \texttt{\emailLe}}
\affil[3]{\affilKien \\ \texttt{\emailKien}}
\affil[*]{\textit{Corresponding Author}}

\begin{document}
\maketitle
\begin{abstract}
    \paperAbstract
\end{abstract}



\section{Introduction}

Support Vector Machines (SVMs)~\citep{cortes1995svm,boser1992training,vapnik1998statistical} have long been regarded as one of the most influential methods in supervised learning, owing to their solid theoretical foundation and strong empirical performance. 
A key property of SVMs is that the resulting decision boundary depends only on a small subset of training samples, known as support vectors, while the remaining samples do not influence the optimal solution.

Despite these appealing properties, training SVMs can be computationally expensive for large-scale datasets, as standard solvers must process the entire training set throughout the optimization procedure. 
This is inherently inefficient, since the optimal solution ultimately depends on only a small number of support vectors. 
However, because the identities of these support vectors are unknown before solving the optimization problem, identifying and removing redundant samples in advance is nontrivial. 
This naturally motivates the development of screening techniques with theoretical guarantees, which aim to safely eliminate irrelevant training samples before or during training.

To address this challenge, safe screening techniques have been developed to identify elements of an optimization problem that are guaranteed to be inactive at optimality and can therefore be safely removed~\citep{elghaoui2010,ogawa2014sample,ogawa2013safe,ndiaye2015gap,ndiaye2016gap,ndiaye2017gap,zhao2014safe,pan2017safe,yang2018safe,guyard2022screen,su2024safe,tran2026new}.

Depending on the target, safe screening can be classified into feature screening~\citep{elghaoui2010}, sample screening~\citep{ogawa2013safe}, and hybrid approaches that combine both~\citep{shibagaki2016simultaneous}. In the context of SVMs, sample screening is particularly important, as it reduces the training set size without affecting the optimal classifier.

Beyond the choice of screening target, safe screening methods can also be distinguished by how the screening rules are constructed and applied during optimization. From a methodological perspective, safe screening strategies can be classified into static~\citep{elghaoui2010}, sequential~\citep{wang2013lasso}, and dynamic~\citep{bonnefoy2014dynamic} approaches, depending on whether screening rules are applied once, along a regularization path, or iteratively refined during optimization.

Another important classification is based on the geometry of the safe region used to localize the optimal solution. In sample screening, the safe region contains a primal solution, whereas in feature screening it contains a dual solution. Common constructions include ball-shaped~\citep{elghaoui2010}, dome-shaped~\citep{tran2022beyond}, ellipsoid-shaped~\citep{dai2012ellipsoid,mialon2020screening}, linear programming based approach~\citep{shibagaki2016simultaneous} and region-free approach~\citep{herzet2022region}. Among these, GAP ball~\citep{fercoq2015mind,ndiaye2015gap,ndiaye2017gap} and its refinement RYU ball~\citep{tran2025one} regions have emerged as a state-of-the-art choice due to their strong theoretical guarantees and practical effectiveness.

Despite their strong guarantees and practical success in classical SVMs, existing safe screening methods have been primarily developed under standard SVM formulations. 
However, classical SVMs are known to be sensitive to noise, which can significantly degrade generalization performance and motivates the development of robust variants.

In robust support vector machines, one typically considers robustness with respect to label noise or feature noise.
In this paper, we focus on robustness against feature noise, where uncertainty is explicitly modeled through uncertainty sets around the training samples.

One line of research addresses robustness to label noise and outliers by adopting non-convex loss functions, such as the ramp loss.
For this class of ramp-loss-based SVMs, safe screening techniques have already been investigated and shown to provide computational benefits~\citep{zhai2019safe}.

A different line of work considers robustness against feature noise by optimizing against worst-case perturbations within uncertainty sets.
In particular, the Robust SVM (R-SVM) proposed by~\cite{xu2009robust} formulates training as a worst-case robust optimization problem and yields classifiers that are more stable under feature perturbations.
However, this robustness comes at the cost of a more challenging optimization problem compared to standard SVMs.

Despite the success of safe screening methods for standard SVMs and ramp-loss-based robust SVMs, to the best of our knowledge, no safe screening framework has been developed for the feature-noise R-SVM model of Xu et al.
This gap stems from the min--max structure induced by uncertainty sets, which cannot be directly handled by existing screening approaches.

Specifically, most existing safe screening frameworks for SVM-type models rely on a composite convex structure of the form 
$
f(Aw) + g(w),
$
and are developed under Fenchel-Rockafellar duality. 
Within this setting, one derives the dual problem, establishes KKT conditions, and constructs screening rules based on a safe region that contains optimal dual solutions. 
Classical constructions such as the GAP safe ball \citep{ndiaye2017gap} and the more recent, tighter RYU safe ball \citep{tran2025one} leverage strong convexity and primal-dual gap bounds inherent to this composite framework.

In contrast, the main loss function in R-SVM involves a min-max formulation induced by uncertainty sets and cannot be expressed in the composite form $f(Aw)$. 
Consequently, the Fenchel-Rockafellar framework underlying existing GAP and RYU constructions does not directly apply. 
To overcome this difficulty, we reformulate the R-SVM as a second-order cone program and derive its dual via Lagrangian duality, from which the KKT conditions and screening rules are established. 
While the RYU ball provides a tighter safe region, its derivation crucially depends on the composite Fenchel-Rockafellar structure and is therefore not compatible with our dual formulation. 
We thus adapt the GAP safe ball within the Lagrangian framework to construct a valid safe region for the robust R-SVM.

To further clarify the positioning of our method within the existing safe screening literature, we summarize the main related approaches in Table~\ref{tab:position_paper}.









\begin{table}[H]
\centering
\caption{Position of our result among closely related safe screening approaches}

\label{tab:position_paper}
\begin{tabular}{l c c c l}
\toprule
\textbf{Problem} \& \textbf{Reference} & 
\textbf{Noise} &
\textbf{Screened} & \textbf{Methodology} & \textbf{Safe region} \\
\midrule
Sparse SVM 
\citep{elghaoui2010} 
& --- 
& Features 
& Static 
& --- \\

SVM 
\citep{ogawa2013safe} 
& --- 
& Samples 
& Sequential 
& Dome \\

Sparse SVM 
\citep{zhao2014safe} 
& --- 
& Features 
& Sequential 
& Dome \\

Sparse SVM 
\citep{shibagaki2016simultaneous} 
& --- 
& Samples, Features 
& Dynamic 
& Ball (GAP) \\


Ramp-loss SVM 
\citep{zhai2019safe} 
& Label
& Samples 
& Dynamic 
& Ball (GAP) \\

R-SVM 
\textbf{(This paper)}
& Feature
& Samples 
& Dynamic 
& Ball (GAP) \\
\bottomrule
\end{tabular}
\end{table}

This paper makes the following contributions. First, we establish the Lagrangian dual formulation of the R-SVM problem and derive its Karush--Kuhn--Tucker (KKT) optimality conditions. Based on these conditions, we develop an ideal safe screening rule and a practical screening rule by adapting the GAP safe ball region to the robust setting. Finally, we demonstrate the effectiveness of the proposed approach through numerical experiments, showing that it significantly accelerates training while preserving classification accuracy.

The source code for reproducing the experiments is publicly available at 
\url{https://github.com/NguynHau/Paper_SS_RSVM_2026/tree/v1.0}.
An archived version of the code is also available on Zenodo with DOI
\href{https://doi.org/10.5281/zenodo.18996767}{10.5281/zenodo.18996767}.

The remainder of the paper is organized as follows. Section~\ref{sec:robustsvm} reviews the R-SVM model. Section~\ref{sec:dual_kkt} presents the dual formulation and the associated KKT conditions. Section~\ref{sec:safescreening} introduces the proposed safe screening rules, and Section~\ref{sec:experiments} reports the experimental results.

\section{R-SVM with Uncertainty sets}
\label{sec:robustsvm}

We consider the binary classification setting, where the goal is to separate two classes of data points using a hyperplane  in a $d$-dimensional Euclidean space \citep{cortes1995svm}. Let
\[
\{(x_i,y_i)\}_{i \in \mathcal{N}}, \quad
    x_i \in \mathbb{R}^d, \quad y_i \in \{-1,+1\}, 
\]
be the training dataset, where $\mathcal{N} \triangleq \{1,\ldots,n\}$ is the index set of training samples, $x_i$ represents the feature vector of the $i$-th sample, and $y_i$ its associated binary label. A linear classifier is a hyperplane parameterized by a weight vector $w \in \mathbb{R}^d$, defining the decision function\footnote{Without loss of generality, the bias term is omitted, since it can be absorbed into $w$ by augmenting each feature vector $x_i$ with an additional constant dimension.}:
\begin{equation*}
    \label{eq:linear_classifier}
    f(x) \triangleq  \langle w, x \rangle,
\end{equation*}
the prediction rule is then $\mathrm{sign}(f(x))$, where the sign indicates the assigned class. We refer to $f(x) = 0$ as the separating hyperplane, while $f(x) = \pm 1$ correspond to the margin hyperplanes. Training samples for which the functional margin satisfies $y_i f(x_i) = 1$ lie on the margin and are known as \emph{support vectors}.


In the classical formulation of SVM \citep{cortes1995svm}, misclassification and margin violations are quantified through the hinge loss, defined as
\begin{equation}
    \label{eq:hinge_loss}
    \ell_i(w) = [1 - y_if(x_i)]_+, \quad \forall i \in \mathcal{N}.
\end{equation}
Here, $[t]_+=\max(t, 0)$. Intuitively, the hinge loss penalizes samples that are either misclassified or correctly classified but lie within the margin, while zero loss is incurred only by samples whose functional margin satisfies $y_i f(x_i) \ge 1$.

The soft-margin SVM then seeks a balance between maximizing the margin and minimizing the cumulative hinge loss \eqref{eq:hinge_loss}, leading to the convex optimization problem.
\begin{equation}
    \label{eq:soft_svm}
    \min_{w \in \mathbb{R}^d} \;\;\mathcal{P}_0(w)\triangleq  \frac{1}{2}\|w\|_2^2 + C \sum_{i \in \mathcal{N}} \ell_i(w),
\end{equation}
where the regularization parameter $C >0$ controls the trade-off between geometric margin maximization and hinge loss minimization. This formulation recovers the hard-margin maximum-margin classifier in the separable case while allowing controlled margin violations in the non-separable setting.

In practice, training data are often contaminated by noise or measurement errors, which can significantly degrade the performance of standard linear classifiers. To explicitly account for such uncertainty, we adopt the Robust Support Vector Machine (R-SVM) framework proposed by \citet{xu2009robust}.

Each observed sample is modeled as $x_i = \tilde{x}_i + \delta_i$, where $\tilde{x}_i \in \mathbb{R}^d$ denotes the nominal input and $\delta_i$ represents an unknown perturbation. Following \citet{xu2009robust}, the perturbation is assumed to belong to a norm-bounded uncertainty set. While the original R-SVM framework allows general $\ell_p$-norm uncertainty sets, for simplicity, in this work we focus on the $\ell_2$-norm case\footnote{The $\ell_1$, $\ell_2$, and $\ell_\infty$ norms induce a cross-polytope, a Euclidean ball, and a hypercube, respectively (\citet{xu2009robust}). We focus on the $\ell_2$ case since it yields a differentiable, and facilitates the derivation of safe screening rules (see Section~\ref{sec:safescreening}).} and define:
\begin{equation}
\label{eq:uncertainty_set}
\mathcal{U}_i = \left\{ \delta_i \in \mathbb{R}^d \;\middle|\; \|\delta_i\|_2 \le \rho_i \right\}, \quad \forall i \in \mathcal{N},
\end{equation}
where $\rho_i \ge 0$ controls the magnitude of the uncertainty. 

Robustness is enforced by requiring the classification constraint to hold for all admissible perturbations in $\mathcal{U}_i$ defined in \eqref{eq:uncertainty_set}, 
the R-SVM problem can be written as
\begin{equation}
\label{eq:firstrobust}
\min_{w \in \mathbb{R}^d} \max_{(\delta_1, \delta_2, \dots, \delta_n) \in \Pi_{i =1}^{n}\mathcal{U}_i} 
\; \frac{1}{2}\|w\|_2^2
+ C \sum_{i \in \mathcal{N}} \bigl[\,1 - y_i \langle w, \tilde{x}_i + \delta_i  \rangle  \bigr]_+.
\end{equation}

Equivalently, the min--max formulation in \eqref{eq:firstrobust} can be written in a compact form as
\begin{equation}
\label{eq:secondrobust}
\min_{w \in \mathbb{R}^d}  
\; \frac{1}{2}\|w\|_2^2
+ C \sum_{i \in \mathcal{N}} \ell_i^{\text{rob}}(w),
\end{equation}
where the robust loss function is defined as $\ell_i^{\text{rob}}(w)  \triangleq \max_{\|\delta_i\|_2 \le \rho_i} \left[1 - y_i \langle w, \tilde{x}_i + \delta_i  \rangle \right]_+. $ 

By exploiting the monotonicity of the hinge loss and properties of norm-bounded uncertainty sets in \eqref{eq:uncertainty_set} can be equivalently reformulated in closed form, resulting in a tractable convex optimization problem.

\begin{lemma}
\label{lem:robust_loss}
Consider a training example $i$, the robust loss admits the following closed-form expression:
\begin{equation}
\label{eq:ell_i}
 \ell_i^{\text{rob}}(w) = \left[ 1 - y_i \langle w, \tilde{x}_i \rangle + \rho_i \|w\|_2 \right]_+.   
\end{equation}

\end{lemma}

\begin{proof}
Since the hinge loss is monotone non-decreasing, the inner maximization over $\delta_i$ can be interchanged with the hinge operator. Hence,
\[
\ell_i^{\text{rob}}(w)
=
\left[
\max_{\|\delta_i\|_2 \le \rho_i}
\Bigl( 1 - y_i \langle w,  \tilde{x}_i \rangle - y_i \langle w, \delta_i \rangle \Bigr)
\right]_+ .
\]
The inner maximization problem is linear in $\delta_i$ and is taken over an
$\ell_2$-ball.
Since $y_i \in \{\pm 1\}$, maximizing $-y_i \langle w, \delta_i \rangle$ is equivalent to
maximizing the absolute inner product $|\langle w, \delta_i \rangle|$.
By the Cauchy--Schwarz inequality, we obtain
\[
\max_{\|\delta_i\|_2 \le \rho_i} |\langle w, \delta_i \rangle|
=
\rho_i \|w\|_2 .
\]
Substituting this result into the above expression yields \eqref{eq:ell_i}.
\end{proof}


\begin{theorem}\label{thm:strongly-convex}
The R-SVM problem can be rewritten in the following closed form:
\begin{equation}
\min_{w \in \mathbb{R}^d}
 \mathcal{P}(w) \triangleq \frac{1}{2}\|w\|_2^2
+ C \sum_{i \in \mathcal{N}}
\Bigl[\,1 - y_i \langle w, \tilde{x}_i \rangle
+ \rho_i \|w\|_2 \Bigr]_+ .
\label{eq:robustprimal}
\end{equation}
Moreover, the objective function of the above problem is strongly convex with respect to the variable $w$.
\end{theorem}

\begin{proof}
The closed-form representation~\eqref{eq:robustprimal} follows immediately
by substituting the expression of the robust loss in Lemma~\ref{lem:robust_loss}
into the objective function of problem~\eqref{eq:secondrobust}.

The function $\frac{1}{2}\|w\|_2^2$ is $1$-strongly convex with respect to the Euclidean norm. For each $i \in \mathcal{N}$, the mapping $w \mapsto [\,1 - y_i \langle w, \tilde{x}_i \rangle + \rho_i \|w\|_2 ]_+$ is convex, since it is the composition of the convex and nondecreasing hinge function with a convex function of $w$. Therefore, $\mathcal{P}(w)$ is the sum of a $1$-strongly convex function and convex functions, and hence is $1$-strongly convex.
\end{proof}




From a geometric perspective, R-SVM expands each training sample into an uncertainty ball of radius $\rho_i$ around its nominal location, representing the set of admissible perturbations. The classifier is therefore required to correctly separate all such uncertainty regions rather than only the observed samples. As a consequence, the resulting decision boundary typically exhibits a reduced geometric margin compared to that of the standard SVM, reflecting the need to ensure correct classification under worst-case perturbations. This geometric interpretation is illustrated in Fig.~\ref{fig:r-svm} through a comparison of the margin hyperplanes produced by standard SVM and R-SVM on linearly separable data.

\begin{figure}[!htpb]
    \centering
    \includegraphics[width=0.6\textwidth]{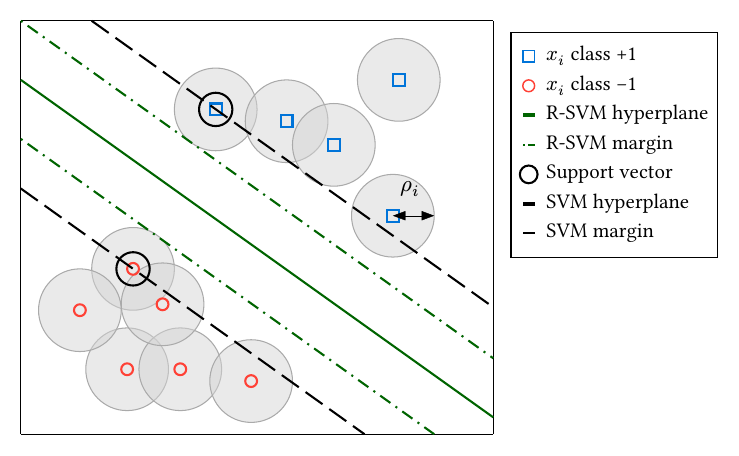}
    \caption{Separating hyperplane of R-SVM on linearly separable data}
    \label{fig:r-svm}
\end{figure}




\begin{remark}
When $\rho_i = 0$ for all $i$, the R-SVM formulation reduces to the classical soft-margin SVM defined in \eqref{eq:soft_svm}.
This observation shows that the R-SVM of Xu et al. can be viewed as a natural generalization of the standard SVM, with the uncertainty radius $\rho_i$ acting as a tunable parameter that controls the degree of robustness against input perturbations.
\end{remark}





\section{Dual Formulation and Optimality Conditions for R-SVM}
\label{sec:dual_kkt}

Most state-of-the-art safe screening methods for large-scale SVMs are fundamentally grounded in duality framework. In particular, they typically rely on expressing the learning problem in a separable form $\min_{w} \, f(Aw) + g(w)$ 
and exploiting Fenchel--Rockafellar duality through the explicit computation of the Fenchel conjugates $f^*$ and $g^*$, see e.g.~\citep{tran2025one}. 
However, R-SVM does not naturally fit this framework. Robustification of the hinge loss introduces an additional term involving the norm of the weight vector inside each hinge loss, thereby breaking the linear structure $Aw$. Even if R-SVM can formally be written in an $f(w)+g(w)$ form by absorbing this nonlinearity into $f$, the resulting function becomes a composition of the hinge loss with both an affine mapping and a norm. As a consequence, deriving the $f^*$ in a tractable closed form becomes nontrivial.

In the original work of Xu et al.~\citep{xu2009robust}, R-SVM is formulated and solved as a convex second-order cone program (SOCP). However, the corresponding dual problem and primal--dual optimality conditions are not explicitly investigated. To remain consistent with this formulation while enabling safe screening, we also adopt an SOCP-based representation of R-SVM and derive its Lagrangian dual directly, rather than relying on Fenchel--Rockafellar duality. This primal--dual characterization forms the theoretical foundation for the safe screening rules developed in the subsequent section.

\subsection{SOCP Reformulation of R-SVM}

The R-SVM, also referred to as the primal problem, can be expressed as a second-order cone program (SOCP) by introducing auxiliary nonnegative variables $t$ and $\xi_i$. Specifically, $t$ upper-bounds the weight norm $\|w\|_2$, and $\xi_i$ controls the hinge loss for each sample under the worst-case perturbation. The resulting formulation is
\begin{equation}
\label{eq:primal_socp}
\begin{aligned}
\min_{\substack{
w \in \mathbb{R}^d,\;  \xi \in \mathbb{R}_+^n,\\ t \in \mathbb{R}_+
}}
\quad &
\frac{1}{2} \|w\|_2^2 + C \sum_{i \in \mathcal{N}} \xi_i \\
\text{s.t.} \quad
& 1 - y_i \langle w,  \tilde{x}_i \rangle + \rho_i t - \xi_i \le 0,
\quad i \in \mathcal{N}, \\
& \xi_i \ge 0, \quad i \in \mathcal{N}, \\
& \|w\|_2 \le t.
\end{aligned}
\end{equation}


\subsection{Dual Problem}

To derive the dual of \eqref{eq:primal_socp}, we introduce the Lagrange multipliers $\alpha_i$, $\mu_i$ (for $i \in \mathcal{N}$), and $\gamma$ corresponding respectively to the functional margin constraints, slack nonnegativity, and the cone constraint $\|w\|_2 \le t$. Under dual feasibility, these multipliers satisfy
\begin{equation}
    \label{eq:dual_feas}
    \alpha_i \ge 0, \quad \mu_i \ge 0, \quad \gamma \ge 0.
\end{equation}
The resulting Lagrangian is given by
\begin{equation}
\label{eq:lagrangian}
\begin{aligned}
\mathcal{L}(w,\xi,t;\alpha,\mu,\gamma) 
= \frac{1}{2}\|w\|_2^2 - \langle d, w \rangle + \gamma \|w\|_2  
+ t\big(s-\gamma\big) + \sum_{i \in \mathcal{N}} \xi_i (C-\alpha_i-\mu_i) + \sum_{i \in \mathcal{N}} \alpha_i,
\end{aligned}
\end{equation}
where
\begin{equation}
\label{eq:dc}
d \triangleq \sum_{i \in \mathcal{N}} \alpha_i y_i \tilde{x}_i, \qquad s \triangleq \sum_{i \in \mathcal{N}} \alpha_i \rho_i.
\end{equation}

Applying the first-order optimality (KKT) conditions to \eqref{eq:lagrangian} yields necessary relations between the primal and dual variables. Differentiating with respect to $b$, $\xi_i$, and $t$ gives 
\begin{subequations}
\label{eq:kkt_stationarity_primal}
\begin{align}
&0 \in \partial_w \mathcal{L} = w - d + \gamma \, \partial \|w\|_2
    &&\Rightarrow 
    w =
        \begin{cases}
        0, & \|d\|_2 \le \gamma,\\
        \big(1 - \frac{\gamma}{\|d\|_2}\big) d , &\|d\|_2 > \gamma,
        \end{cases}
    \tag{13a}\label{eq:kkt_w}\\
&\frac{\partial \mathcal{L}}{\partial \xi_i} = C - \alpha_i - \mu_i = 0 
    &&\Rightarrow \mu_i = C - \alpha_i \ge 0 
    \Rightarrow 0 \le \alpha_i \le C,  
    \tag{13b}\label{eq:kkt_xi}\\
&\frac{\partial \mathcal{L}}{\partial t} = s - \gamma = 0 
    &&\Rightarrow \gamma = s = \sum_{i \in \mathcal{N}} \alpha_i \rho_i.
    \tag{13c}\label{eq:kkt_t}
\end{align}
\end{subequations}


According to \eqref{eq:kkt_w}, at the optimal solutions the vector $w^\star$ satisfies
\begin{equation}
\label{eq:w}
w^\star
=\Big(1-\dfrac{s}{\| d \|_2}\Big)d,
\qquad \text{for } \|d\|_2>\gamma.
\end{equation}

Minimizing the Lagrangian in~\eqref{eq:lagrangian} over the primal variables  $(w,\xi,t)$ yields the dual function
\begin{equation}
\label{eq:dual_func}
\mathcal{D}(\alpha)
= \sum_{i \in \mathcal{N}} \alpha_i
  + \inf_{w}\,\Phi(w),
\end{equation}
where \(\Phi(w) \triangleq \frac{1}{2}\|w\|_2^2 - \langle d, w \rangle + s\|w\|_2.\) We now derive the closed-form dual of the R-SVM primal problem \eqref{eq:robustprimal} using the above KKT conditions and the primal \eqref{eq:robustprimal}.

Using these KKT insights, we now derive the closed-form dual.

\begin{theorem}
The dual problem of the R-SVM is given by
\begin{equation}
\label{eq:dual_problem}
\alpha^\star
\in
\arg\max_{\alpha \in [0,C]^n}
\left\{
\sum_{i \in \mathcal{N}} \alpha_i
-
\frac{1}{2}
\left[
\left\|
\sum_{i \in \mathcal{N}} \alpha_i y_i \tilde{x}_i
\right\|_2
-
\sum_{i \in \mathcal{N}} \alpha_i \rho_i
\right]_+^2
\right\}.
\end{equation}
\end{theorem}

\begin{proof}
We consider the following subproblem:
\begin{equation}
\inf_{w \in \mathbb{R}^d} \Phi(w)
=
\frac{1}{2}\|w\|_2^2 - \langle d, w \rangle + s \|w\|_2 .
\label{eq:infphi}
\end{equation}

By the Cauchy--Schwarz inequality, we have $\langle d, w \rangle \le \|d\|_2 \|w\|_2$. It follows that
\[
\Phi(w)
\ge
f(\tau)
\triangleq
\frac{1}{2}\tau^2 - \|d\|_2 \tau + s \tau,
\quad
\text{where }
\tau \triangleq \|w\|_2 \ge 0 .
\]

Equality holds when \( w \) is in the same direction as \( d \). Therefore, minimizing \( \Phi(w) \) over problem \eqref{eq:infphi} is equivalent to minimizing the one-dimensional function \( f(\tau) \) over the domain \( \tau \ge 0 \), i.e.,
\[
\inf_{w \in \mathbb{R}^d} \Phi(w)
=
\inf_{\tau \in \mathbb{R}_+} f(\tau).
\]

Since \( f(\tau) \) is a quadratic function of one variable defined on \( [0, +\infty) \), its minimum can be determined explicitly. In particular, the derivative of \( f(\tau) \) is given by $f'(\tau) = \tau - (\|d\|_2 - s)$, with stationary point $\tau^\star = \|d\|_2 - s$.
We now consider the following two cases:
\begin{itemize}
    \item If \( \|d\|_2 \le s \), then \( \tau^\star \le 0 \), and hence the optimum over the domain \( \tau \ge 0 \) is attained at \( \tau = 0 \),
        \[
        \inf_{w \in \mathbb{R}^d} \Phi(w)
        =
        f(0)
        =
        0 .
        \]
    \item If \( \|d\|_2 > s \), then \( \tau^\star = \|d\|_2 - s > 0 \). In this case,
        \begin{align*}
            \inf_{w \in \mathbb{R}^d} \Phi(w)
            =
            f(\tau^\star) 
            &=
            \frac{1}{2} (\tau^\star)^2
            -
            \bigl[\|d\|_2 \tau^\star - s \tau^\star \bigr] 
            =
            -\frac{1}{2} (\|d\|_2 - s)^2 .
        \end{align*}
\end{itemize}

Substituting the two cases above into \eqref{eq:dual_func}, the dual function can be written in the following form:
\[
\mathcal{D}(\alpha)
=
\begin{cases}
1^\top \alpha,
& \text{if } \|d\|_2 \le s, \\[0.5em]
1^\top \alpha
-
\frac{1}{2} (\|d\|_2 - s)^2,
& \text{if } \|d\|_2 > s .
\end{cases}
\]

Combining the two cases, we obtain
\begin{equation}
\mathcal{D}(\alpha)
=
1^\top \alpha
-
\frac{1}{2} [\|d\|_2 - s]_+^2,
\label{eq:dR-SVMfinal}
\end{equation}
where \( d \) and \( s \) are defined in \eqref{eq:dc}.

Finally, maximizing the dual function \eqref{eq:dR-SVMfinal} subject to the constraints \eqref{eq:kkt_xi}, we obtain the dual problem of the R-SVM:
\[
    \max_{\alpha \in \mathcal{Z}}
    \mathcal{D}(\alpha), \quad \mathcal{Z}  \left\{ \alpha \in \mathbb{R}^n \;\middle|\; 0 \le \alpha_i \le C \right\}.
\]

Since the feasible set \( \mathcal{Z} \) is compact and the dual function \eqref{eq:dR-SVMfinal} is continuous on this set, the dual problem always admits at least one optimal solution.

However, since the dual function \eqref{eq:dR-SVMfinal} is not strictly concave and is nondifferentiable at the optimum, the set of dual maximizers is not
necessarily a singleton. Therefore, the dual solution is described as a set of
optimal solutions rather than a unique point, which justifies the notation used
in \eqref{eq:dual_problem}.
\end{proof}

\subsection{Slater's Condition and Strong Duality}

Consider the primal SOCP~\eqref{eq:primal_socp} and its dual problem~\eqref{eq:dual_problem}.
Since the primal problem is convex and the dual is derived via the Lagrangian,
weak duality holds.
In particular, for any primal feasible point $(w,\xi,t)$ and any dual feasible
vector $\alpha$, we have
\begin{equation}
\label{eq:weak_duality}
\mathcal{D}(\alpha)
\;\le\;
\frac{1}{2}\|w\|_2^2 + C \sum_{i \in \mathcal{N}} \xi_i ,
\end{equation}
showing that the dual objective provides a global lower bound on the primal
objective value.

To establish strong duality, we verify Slater’s condition for the primal problem.
Consider $\bar{w} = 0$, $\bar{t} > 0$ and sufficiently large $\bar{\xi}_i$ such that $1 + \rho_i \bar{t} - \bar{\xi}_i < 0$, $\forall i \in \mathcal{N}$.
Then all constraints of \eqref{eq:primal_socp} are strictly satisfied:
\[
\begin{aligned}
&1 - y_i\langle \bar{w}, \tilde{x}_i \rangle + \rho_i \bar{t} - \bar{\xi}_i < 0, \quad \forall i \in \mathcal{N}, \\
&\bar{\xi}_i > 0, \quad \forall i \in \mathcal{N}, \\
&\|\bar{w}\|_2 = 0 < \bar{t}.
\end{aligned}
\]
Hence, Slater’s condition holds for the primal problem.



\subsection{Karush--Kuhn--Tucker Conditions}
\label{subsec:kkt}

Under strong duality, the KKT conditions provide a complete characterization of optimality for the R-SVM.
Let $(w^\star,\xi^\star,t^\star)$ and $(\alpha^\star,\mu^\star,\gamma^\star)$ denote a primal--dual optimal solution pair.

The stationarity conditions with respect to the primal variables $(w,\xi,t)$, together with dual feasibility, have already been derived in  \eqref{eq:dual_feas} and \eqref{eq:kkt_stationarity_primal}.

In addition, primal feasibility requires
\begin{subequations}
\label{eq:kkt_primal_feas_recall}
\begin{align}
y_i \langle w^{\star}, \tilde{x}_i \rangle - \rho_i t^\star + \xi_i^\star &\ge 1, \label{eq:kkt_pf_a}\\
\xi_i^\star &\ge 0, \label{eq:kkt_pf_b}\\
t^\star - \|w^\star\|_2 &\ge 0. \label{eq:kkt_pf_c}
\end{align}
\end{subequations}

The complementary slackness conditions are given by
\begin{subequations}
\label{eq:kkt_complementarity}
\begin{align}
\alpha_i^\star
\big(
y_i \langle w^{\star}, \tilde{x}_i \rangle
- \rho_i t^\star
+ \xi_i^\star
- 1
\big) &= 0, \label{eq:kkt_cs_a}\\
\mu_i^\star \xi_i^\star &= 0, \label{eq:kkt_cs_b}\\
\gamma^\star (t^\star - \|w^\star\|_2) &= 0. \label{eq:kkt_cs_c}
\end{align}
\end{subequations}

Together, these conditions determine the exact primal--dual structure of the optimal solution.
In particular, they allow the dual status of each sample to be characterized explicitly, which forms the foundation of the safe screening rules developed in the next section.


\section{Safe Screening for R-SVM}
\label{sec:safescreening}

Safe screening aims to accelerate large-scale SVM training by identifying, in advance, training samples that are guaranteed to be either correctly classified or misclassified by margin hyperplanes. This reduces the dual problem size and computational complexity while preserving optimality. In the following, we first present an ideal safe screening rule, followed by a practical variant that relies on the construction of a GAP ball.

\subsection{Ideal Safe Screening}
\label{sec:ideal_safe}


We first observe the relationship between the optimal dual solution $ \alpha^\star $ and the primal variables $ w^\star $. Specifically, if $ \alpha_i^\star = 0 $, then the $ i $-th sample, including the perturbation, is correctly classified by the margin hyperplane, i.e., $y_i  \langle w^\star, \tilde{x}_i \rangle 
- \rho_i \| w^\star \|_2 \geq 1$.
On the other hand, if $ \alpha_i^\star = C $, then the $ i $-th sample, including the perturbation, is guaranteed to be misclassified by the margin hyperplane. Further details of this relationship are established in the following lemma.

\begin{lemma}
\label{lem:dual_to_margin}
At the optimal solution $w^\star$ of the R-SVM, define
\[
\psi_i^\star \triangleq y_i\langle w^\star, \tilde{x}_i \rangle
- \rho_i \|w^\star\|_2 .
\]
Then the optimal dual variable $\alpha_i^\star$ satisfies
\[
\alpha_i^\star
\begin{cases}
= 0 
&\Rightarrow\ \psi_i^\star \ge 1, \\[0.4em]
\in (0,C) 
&\Rightarrow\ \psi_i^\star = 1, \\[0.4em]
= C 
&\Rightarrow\ \psi_i^\star \le 1 .
\end{cases}
\]
\end{lemma}

\begin{proof}
From the stationarity condition~\eqref{eq:kkt_xi}, we have
$\mu_i^\star = C - \alpha_i^\star \ge 0$.
Together with primal feasibility and complementarity slackness
conditions~\eqref{eq:kkt_pf_a}--\eqref{eq:kkt_cs_b},
this yields the following characterization.

\begin{itemize}
\item If $\alpha_i^\star = 0$, then $\mu_i^\star = C > 0$ and
$\xi_i^\star = 0$ by~\eqref{eq:kkt_cs_b}.
Primal feasibility~\eqref{eq:kkt_pf_a} implies $\psi_i^\star \ge 1$.

\item If $0 < \alpha_i^\star < C$, then $\mu_i^\star > 0$ yields
$\xi_i^\star = 0$, and complementarity slackness~\eqref{eq:kkt_cs_a}
enforces $\psi_i^\star = 1$.

\item If $\alpha_i^\star = C$, then $\mu_i^\star = 0$ and, by primal feasibility implies $\xi_i^\star \ge 0$. With~\eqref{eq:kkt_cs_a}, this yields $\psi_i^\star \le 1$.
\end{itemize}

Finally, since the constraint $\|w\|_2 \le t$ is tight at optimality without loss of generality, $\psi_i^\star$ is well defined, and the stated relationships between $\alpha_i^\star$ and $\psi_i^\star$ follow.
\end{proof}

The preceding lemma clarifies the relationship between the optimal dual variables $\alpha_i^\star$ and $\psi_i^\star$ at the R-SVM optimum. In particular, $\psi_i^\star = 1$ corresponding to support vectors lying on the margin, while correctly classified samples with $\psi_i^\star > 1$ and misclassified samples with $\psi_i^\star < 1$ allow one to determine the values of the corresponding optimal dual variables. When many dual variables $\alpha_i^\star$ can be identified in this way, the size of the dual problem can be significantly reduced. This forms the basis of the ideal safe sample screening rule stated in the following theorem.

\begin{theorem}[Ideal safe sample screening]
\label{thm:ideal_safe}
At the optimal solution $w^\star$ of the R-SVM, the following
implications hold for each training sample $i$:
\[
\begin{aligned}
&\psi_i^\star > 1 \quad \Rightarrow \quad \alpha_i^\star = 0, 
\\
&\psi_i^\star < 1 \quad \Rightarrow \quad \alpha_i^\star = C.
\end{aligned}
\]
\end{theorem}

\begin{proof}
The result follows directly from Lemma~\ref{lem:dual_to_margin}. 
\end{proof}

Theorem~\ref{thm:ideal_safe} shows that the boundary value of $\alpha_i^\star$ can be determined by the $\psi_i^\star$, which in turn depends on the exact optimal primal solution $w^\star$.
However, identifying $\psi_i^\star$ exactly requires solving the optimization problem to full optimality.

In practice, instead of accessing $w^\star$ directly, we construct a region that is guaranteed to contain the optimal primal solution and use it to estimate $\psi_i^\star$ in a worst-case sense. If a sample lies strictly on one side of the classification margin for all $w$ in this safe region, then its corresponding dual status can be determined with certainty without solving the problem to completion.

\subsection{Practical Safe Screening Rule}

\label{sec:practical_safe}

We derive a practical safe screening rule for the R-SVM based on the primal--dual structure and strong convexity of the objective, following the GAP ball framework~\citep{ndiaye2017gap}. The construction relies on bounding the distance between the current iterate and the optimal solution.


\begin{lemma}[Strong convexity]
\label{lem:strongconvexity}
Consider the primal R-SVM objective $\mathcal{P}(w)$ defined in~\eqref{eq:robustprimal}. Then $\mathcal{P}(w)$ is $1$-strongly convex with respect to $w$, and the optimal solution $w^\star$ satisfies
\begin{equation}
\label{eq:strongconvexity_ineq}
\mathcal{P}(w) \ge \mathcal{P}(w^\star)
+ \frac{1}{2}\|w - w^\star\|_2^2,
\quad \forall w \in \mathbb{R}^d.
\end{equation}
\end{lemma}

\begin{proof}
The quadratic regularization term $\frac{1}{2}\|w\|_2^2$ is $1$-strongly convex with respect to $w$, while each robust hinge loss $\ell_i^{\rm rob}$ is convex in $w$ by assumption. Therefore, their sum $\mathcal{P}(w)$ is $1$-strongly convex with respect to $w$.

By the definition of strong convexity, for any $w$ and any subgradient
$g^\star \in \partial_w \mathcal{P}(w^\star)$, we have
\[
\mathcal{P}(w)
\ge
\mathcal{P}(w^\star)
+ \langle g^\star, w - w^\star \rangle
+ \frac{1}{2}\|w - w^\star\|_2^2.
\]
At the optimal solution, the first-order optimality condition yields
$0 \in \partial_w \mathcal{P}(w^\star)$,  so we may choose $g^\star = 0$, which implies $\langle g^\star, w - w^\star \rangle = 0$.
Substituting this into the above inequality yields~\eqref{eq:strongconvexity_ineq}.
\end{proof}

Lemma~\ref{lem:strongconvexity} provides a quadratic lower bound on the primal
objective around the optimum, which allows us to control the distance between
any candidate solution and $w^\star$.
By combining this property with the primal--dual gap, we obtain a computable
region guaranteed to contain the optimal solution, known as the GAP ball. Here, we denote by $\Theta(c, R)$ the ball with center $c \in \mathbb{R}^d$ and radius $R \geq 0$.

\begin{lemma}[GAP ball]
\label{thm:gapsafeball}
At iteration $k$ of some iterative solving method, let the current primal solution be $w^{(k)}$ and the dual solution be $\alpha^{(k)}$. The duality gap is defined as $\text{G}^{(k)} \triangleq \mathcal{P}(w^{(k)}) - \mathcal{D}(\alpha^{(k)})$. Then the optimal solution $w^\star$ always lies in the Euclidean ball
\begin{equation}
\label{eq:radius}
w^\star \in \Theta^{(k)} \triangleq \Theta\Big(w^{(k)}, R^{(k)} \Big), \quad 
R^{(k)} \triangleq \sqrt{2 \cdot \text{G}^{(k)}}.
\end{equation}
\end{lemma}

\begin{proof}
Consider the current primal iterate $w^{(k)}$ at iteration $k$, recovered from the dual via \eqref{eq:w}. From Lemma~\ref{lem:strongconvexity}, we have \(
\frac{1}{2} \|w^{(k)} - w^\star\|_2^2 \le \mathcal{P}(w^{(k)}) - \mathcal{P}(w^\star). 
\) 
Moreover, by strong duality, we have $\mathcal{P}(w^\star)=\mathcal{D}(\alpha^\star)$. Since $\mathcal{D}(\alpha^{(k)}) \le \mathcal{D}(\alpha^\star)$, it follows that
$
\mathcal{P}(w^{(k)}) - \mathcal{P}(w^\star) \le \mathcal{P}(w^{(k)}) - \mathcal{D}(\alpha^{(k)}) = \text{G}^{(k)}
$.
Combining the above inequalities, we obtain
\begin{equation*}
\frac{1}{2} \|w^{(k)} - w^\star\|_2^2 \le \text{G}^{(k)} \implies
\|w^\star - w^{(k)}\|_2 \le R^{(k)} = \sqrt{2 \cdot \text{G}^{(k)}}.
\end{equation*}

This shows that the optimal solution $w^\star$ is guaranteed to lie within the
GAP ball $\Theta^{(k)}$ defined in~\eqref{eq:radius}.
\end{proof}



Given the safe region $\Theta^{(k)}$, we now derive deterministic bounds on $\psi_i^\star$. Since $w^\star$ is unknown but satisfies $w^\star \in \Theta^{(k)}$, it suffices to bound the margin uniformly over this region.

\begin{lemma}[Safe Bounds on $\psi_i^\star$]
\label{thr:rob_bounds}
Define
\begin{align}
    \label{eq:lrob}
    \text{LB}_i &\triangleq y_i \langle w^{(k)}, \tilde{x}_i \rangle
 - \rho_i \big( \|w^{(k)}\|_2 + R^{(k)} \big) - R^{(k)} \|\tilde{x}_i\|_2 \in \mathbb{R}, \\
    \label{eq:urob}
    \text{UB}_i &\triangleq y_i \langle w^{(k)}, \tilde{x}_i \rangle  - \rho_i \big[ \|w^{(k)}\|_2 - R^{(k)} \big]_+ + R^{(k)} \|\tilde{x}_i\|_2 \in \mathbb{R},
\end{align}
Consequently, for \( w^\star \in \Theta^{(k)} \),
\begin{equation}
     \label{eq:rob_margin_interval}
     \psi_i^\star \in [\,\text{LB}_i, \text{UB}_i\,], \quad \forall i \in \mathcal{N}.  
\end{equation}
\end{lemma}

\begin{proof}
By Lemma~\ref{thm:gapsafeball}, the optimal solution satisfies $w^\star \in \Theta^{(k)}$. Therefore, it suffices to bound $\psi_i$ uniformly over $\Theta^{(k)}$. We decompose the $\psi_i^\star$ as
\(
\psi_i^\star
=
y_i\langle w^\star, \tilde{x}_i \rangle
- \rho_i \|w^\star\|_2,
\)
and bound each term separately. First, by linearity of the inner product, we have
\[
y_i \langle w^\star, \tilde{x}_i \rangle
= y_i \langle w^{(k)}, \tilde{x}_i \rangle + y_i \langle w^\star - w^{(k)}, \tilde{x}_i  \rangle.
\]
By the Cauchy--Schwarz inequality, $\big| y_i \langle w^\star - w^{(k)}, \tilde{x}_i  \rangle \big|
\le   \|w^\star - w^{(k)}\|_2 \, \|\tilde{x}_i\|_2
\le R^{(k)} \|\tilde{x}_i\|_2.$
Therefore,
\begin{equation}
\label{eq:ywx}
    y_i \langle w^\star, \tilde{x}_i \rangle
    \in
    \Big[
    y_i \langle w^{(k)}, \tilde{x}_i \rangle  - R^{(k)} \|\tilde{x}_i\|_2,
    \;
    y_i \langle w^{(k)}, \tilde{x}_i \rangle + R^{(k)} \|\tilde{x}_i\|_2
    \Big].
\end{equation}

We next consider the norm term $\rho_i \|w^\star\|_2$ in $\psi_i^\star$. By the triangle inequality,
\[
\bigl| \|w^\star\|_2 - \|w^{(k)}\|_2 \bigr| \le \|w^\star - w^{(k)}\|_2 \le R^{(k)}.
\]
It follows that $\max_{w \in \Theta^{(k)}} \|w\|_2 = \|w^{(k)}\|_2 + R^{(k)}$, and $\min_{w \in \Theta^{(k)}} \|w\|_2 = \bigl[ \|w^{(k)}\|_2 - R^{(k)} \bigr]_+$\footnote{
Since $\|w\|_2$ is always nonnegative. Thus, when the bound $\|w^{(k)}\|_2 - R^{(k)}$ becomes negative, it is truncated to zero. Geometrically, this corresponds to the distance from the origin to the ball $\Theta^{(k)}$.
}.
Hence,
\begin{equation}
    \label{eq:rhow}
    \rho_i \|w^\star\|_2 \in
    \Bigl[
    \rho_i \bigl[ \|w^{(k)}\|_2 - R^{(k)} \bigr]_+,
    \;
    \rho_i \bigl( \|w^{(k)}\|_2 + R^{(k)} \bigr)
    \Bigr].
\end{equation}

Substituting \eqref{eq:ywx} and \eqref{eq:rhow} into the definition of $\psi_i^\star$, we directly obtain \eqref{eq:lrob} and \eqref{eq:urob} as lower and upper bound for $\psi_i^\star$, hence
\eqref{eq:rob_margin_interval}.
\end{proof}

Lemma~\ref{thr:rob_bounds} establishes a provably valid interval $[\,\text{LB}_i, \text{UB}_i\,]$ for the unknown value of $\psi_i^\star$ for each sample, without requiring access to the optimal solution $w^\star$. This interval characterization is central to the design of \emph{safe screening rules} for R-SVM training, as it enables early identification of samples whose contribution to the objective is guaranteed to be inactive under the current iterate $w^{(k)}$.

\begin{theorem}[Practical Safe Screening Rules]
\label{thm:safe_screening_rob}
Based on the safe bounds of $\psi_i^\star$ over the GAP ball,
the optimal dual variables $\alpha_i^\star$ for R-SVM satisfy:
\begin{equation}
    \begin{cases}
        \text{LB}_i > 1 & \Rightarrow \alpha_i^\star = 0, \\ 
        \text{UB}_i < 1 & \Rightarrow \alpha_i^\star = C, \\
        \text{otherwise}
            &\Rightarrow \alpha_i^\star \in [0,C].
    \end{cases}
\end{equation}


\end{theorem}

\begin{proof}
The conclusion follows from Theorem~\ref{thm:ideal_safe}, which relates
$\alpha_i^\star$ to the position of $\psi_i^\star$ relative to~$1$, together with Lemma~\ref{thr:rob_bounds}, which provides safe bounds on $\psi_i^\star$. If the interval lies entirely above or below~$1$, the value of $\alpha_i^\star$ is uniquely determined; otherwise, no definitive conclusion can be drawn.
\end{proof}

Accordingly, the training set can be partitioned into three disjoint sets based on the safe bounds:
\begin{subequations}
    \begin{align*}
    \mathcal{R} &\triangleq \{ i \in  \mathcal{N} \mid \text{LB}_i > 1 \}, \\[0.5em]
    \mathcal{S} &\triangleq \{ i \in  \mathcal{N} \mid \text{UB}_i < 1 \}, \\[0.5em]
    \mathcal{F} &\triangleq \mathcal{N} \setminus (\mathcal{R} \cup \mathcal{S}).
    \end{align*}
\end{subequations}
Finally, the \emph{active set} for the optimization problem can be safely restricted to
\begin{equation*}
    \mathcal{A} = \mathcal{S} \cup \mathcal{F} = \mathcal{N} \setminus \mathcal{R},
\end{equation*}

This restriction can substantially reduce the computational burden
without sacrificing optimality guarantees, since all excluded samples
are certified to be either fully inactive ($\alpha_i^\star = 0$)
or saturated ($\alpha_i^\star = C$).





Next, we propose a dynamic safe screening procedure for R-SVM, with detailed pseudocode provided in Algorithm~\ref{alg:ds_robust_svm}. Unlike static screening methods~\citep{elghaoui2010}, which eliminate samples only once before optimization, the proposed approach applies screening rules repeatedly during the optimization process~\citep{bonnefoy2014dynamic}. This dynamic strategy allows the algorithm to progressively identify samples $i$ whose optimal dual variables satisfy $\alpha_i^\star=1$ or $\alpha_i^\star=C$, thereby reducing the effective problem size during training. The algorithm will be used in the next section to evaluate the performance of the proposed screening approach.

\begin{algorithm}[!htbp]
\caption{Dynamic Safe Screening for Robust SVM}
\label{alg:ds_robust_svm}
\begin{algorithmic}[1]

\Require Training set $\{(\tilde{x}_i, y_i, \rho_i)\}_{i\in\mathcal N}$,
parameter $C$,
tolerance duality $\epsilon$,
minimum working set size $F_{\min}$


\State $\mathcal R \gets \emptyset$, $\mathcal S \gets \emptyset$, $\mathcal F \gets \mathcal N$

\State $k \gets 0$, $\alpha^{(0)} \gets 0$, $w^{(0)} \gets 0$, $G^{(0} \gets +\infty$

\While{$|\mathcal F| > F_{\min}$ \textbf{and} $G > \epsilon$}


    \State $w^{(k)}, \alpha^{(k)} \gets \text{train R-SVM with low tolerance on dataset } (\tilde{x}_i, y_i, \rho_i)  \text{ for } i\in\mathcal{F} \text{ and parameter } C$ 

    

    \State Compute primal objective $P^{(k)}$ by \eqref{eq:robustprimal}, dual objective $D^{(k)}$ by \eqref{eq:dual_problem} and duality gap $G^{(k)} = P^{(k)} - D^{(k)}$

    \State Compute the radius of GAP safe ball $R^{(k)}$ using Eq.~\eqref{eq:radius}

    \For{each sample $i \in \mathcal F$}

        \State Compute lower bound $\mathrm{LB}_i$ using Eq.~\eqref{eq:lrob} and upper bound $\mathrm{UB}_i$ using Eq.~\eqref{eq:urob}

        \If{$\mathrm{LB}_i > 1$}
            \State $\mathcal R \gets \mathcal R \cup \{i\}$, $\mathcal F \gets \mathcal F \setminus \{i\}$
        \ElsIf{$\mathrm{UB}_i < 1$}
            \State $\mathcal S \gets \mathcal S \cup \{i\}$, $\mathcal F \gets \mathcal F \setminus \{i\}$
        \EndIf
    \EndFor

    \State $k \gets k + 1$

\EndWhile

\State $\mathcal A \gets \mathcal F \cup \mathcal S$

\State $w^\star \gets \text{train R-SVM with dataset } (\tilde{x}_i, y_i, \rho_i)  \text{ for } i\in\mathcal{A} \text{ and parameter } C$ 

\State \Return $w^\star$

\end{algorithmic}
\end{algorithm}

\section{Experiments}
\label{sec:experiments}

\subsection{Experimental Setup}

All experiments were conducted using Python~3 on the Google Colab platform to evaluate the computational performance of the proposed method in the Robust Support Vector Machine (Robust SVM) setting. The primary objective of the experiments is to examine the impact of safe screening on the optimization process.

In this study, two approaches are considered:

\begin{itemize}
    \item R-SVM (baseline): the R-SVM formulation is expressed as a convex optimization problem and solved using the SCS solver through cvxpy\footnote{Python convex optimization library}.    
    \item R-SVM + Safe screening: an extended version of the model in which a safe screening mechanism is incorporated to eliminate irrelevant features prior to or during the optimization process.
\end{itemize}
The experiments focus on observing the computational behavior of the screening strategy. In particular, we monitor
\begin{itemize}
    \item The proportion of features eliminated by the screening rule,
    \item The overall time required to solve the optimization problem.
\end{itemize}

The experiments are conducted on two real-world datasets. A summary of the datasets is provided in Table~\ref{tab:dataset}.
\begin{table}[!htbp]
\centering
\caption{Summary of the real-world datasets used in the experiments}
\label{tab:dataset}
\begin{tabular}{ccccc}
\toprule
\textbf{Dataset} & \textbf{Classes} & \textbf{Samples} & \textbf{Dimension} & \textbf{Source} \\
\midrule
Breast Cancer Wisconsin & Benign vs Malignant & 569 & 30 & UCI ML Repository \\
Spambase Email & Spam vs Non-spam & 4601 & 57 & UCI ML Repository \\
\bottomrule
\end{tabular}
\end{table}
To study the behavior of the algorithm under different conditions, the experiments are performed with multiple parameter configurations. In particular, four values of the regularization parameter $C$ and four values of the uncertainty radius $\rho$ are considered, where $\rho$ represents the radius of the uncertainty set in the robust formulation.
\[
C \in \{0.01,\,0.1,\,1.0,\,10.0\},
\]
\[
\rho \in \{0.0,\,0.01,\,0.02,\,0.05\}.
\]
For each dataset, experiments are conducted on all combinations of $C$ and $\rho$. This setting allows us to observe how the effectiveness of safe screening and the feature elimination speed vary with different levels of regularization and uncertainty.

Because Python does not provide a dedicated library for solving Robust SVM in the same way that LinearSVC is available for standard SVM, the R-SVM model in this study is formulated explicitly as a convex optimization problem and solved using the SCS solver through CVXPY. Due to the characteristics of the SCS solver, the runtime of the algorithm is not fully deterministic, and the library does not provide a mechanism to fix the random seed to ensure identical runtime across executions. As a result, the execution time of the same experiment may vary slightly between runs. To obtain more reliable experimental results, each parameter configuration is repeated 100 times. The reported results are presented as the mean of the runtime over these runs.

\subsection{Breast Cancer Dataset}

\begin{figure}[!htpb]
    \centering
    \includegraphics[width=\textwidth]{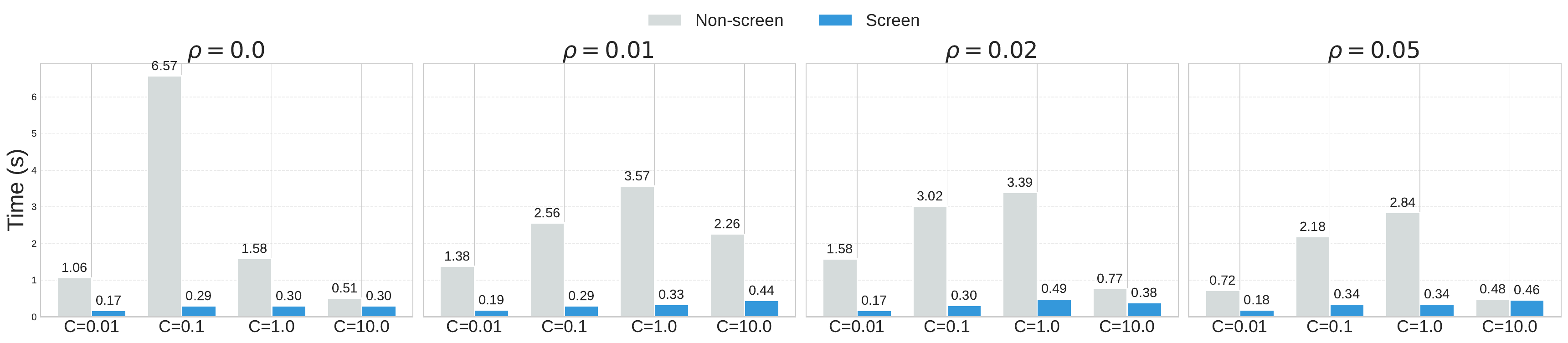}
    \caption{Training time comparison with and without safe screening for different values of $C$ and $\rho$ on the Breast Cancer Wisconsin dataset.}
    \label{fig:time_cancer}
\end{figure}

Gap Safe Screening demonstrates clear acceleration on the Breast Cancer dataset (a small dataset). Specifically, the screening-based method is 1.07 to 18.96 times faster than the baseline R-SVM without screening in terms of training time, as shown in Figure \ref{fig:time_cancer}.

\begin{figure}[!htpb]
    \centering
    \includegraphics[width=\textwidth]{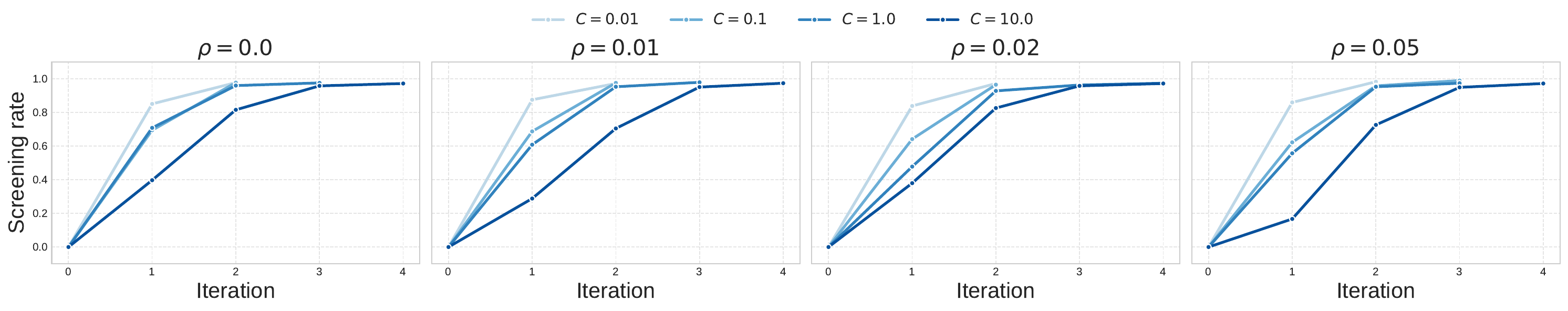}
    \caption{Iteration-wise dynamics of the safe screening elimination ratio for different values of the uncertainty radius $\rho$ on the Breast Cancer Wisconsin dataset.}
    \label{fig:screeningrate_cancer}
\end{figure}

In addition, the results indicate a very high feature screening rate (approximately 96.5\%–98.9\%). Figure \ref{fig:screeningrate_cancer} clearly illustrates this trend, showing that a large portion of features are progressively screened as the optimization algorithm proceeds.

\subsection{Spambase Email Dataset}

\begin{figure}[!htpb]
    \centering
    \includegraphics[width=\textwidth]{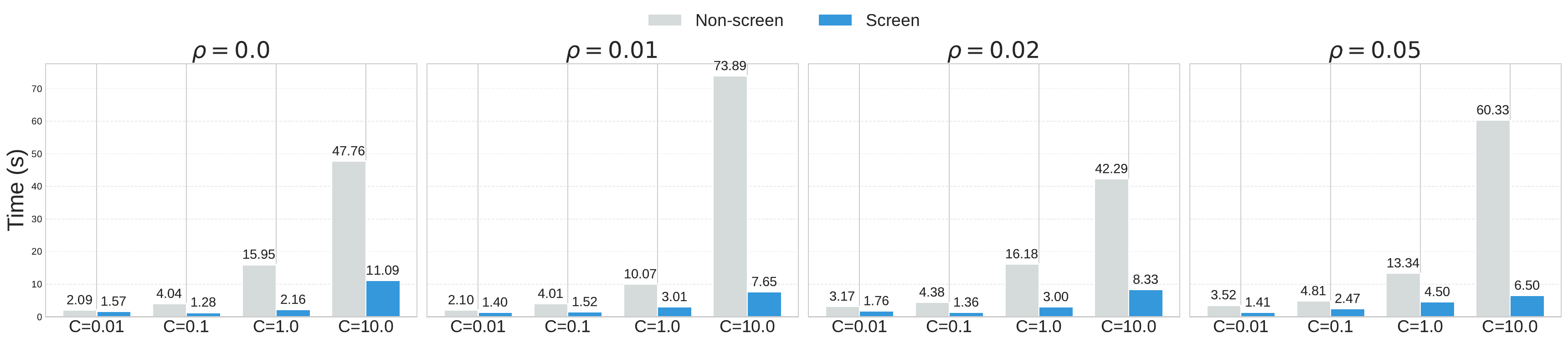}
    \caption{Training time comparison with and without safe screening for different values of $C$ and $\rho$ on the Spambase Email dataset.}
    \label{fig:time_spam}
\end{figure}

On the Spambase Email dataset, which represents a medium-sized dataset, Gap Safe Screening continues to provide notable computational improvements. When the screening strategy is applied, the prediction time is reduced by a factor of 1.54 to 9.85 compared with the baseline R-SVM without screening, as illustrated in Figure \ref{fig:time_spam}.

\begin{figure}[!htpb]
    \centering
    \includegraphics[width=\textwidth]{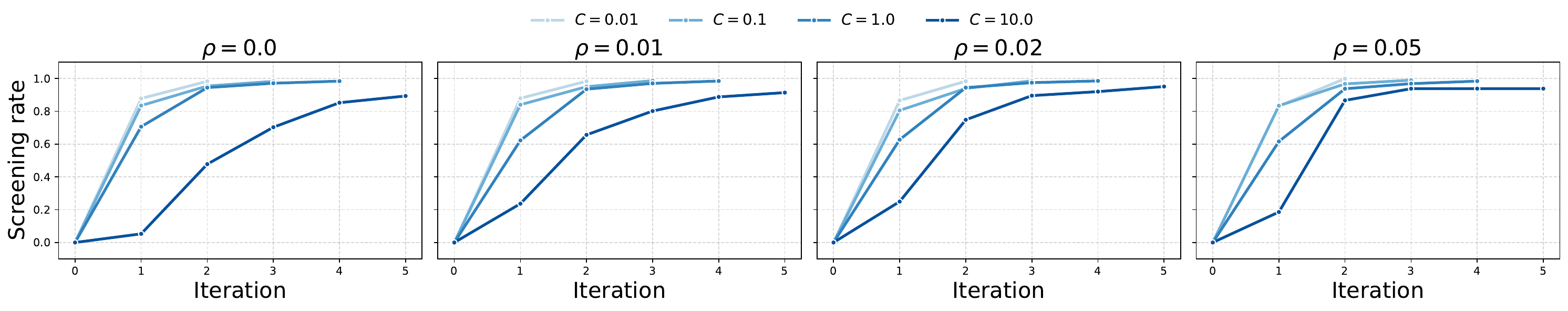}
    \caption{Iteration-wise dynamics of the safe screening elimination ratio for different values of the uncertainty radius $\rho$ on the Spambase Email dataset.}
    \label{fig:screeningrate_spam}
\end{figure}

Moreover, the experimental results show that a substantial proportion of features can be screened out, ranging from approximately 89.3\% to 99.8\% across different parameter settings. Figure \ref{fig:screeningrate_spam} further demonstrates this pattern, indicating that many features can be discarded early before solving the optimization problem, which contributes to the reduction in overall computational cost.
\section{Conclusion}

In this paper, we proposed a safe sample screening framework tailored to R-SVM that explicitly handle feature uncertainty through uncertainty sets. Our approach leverages the Lagrangian dual formulation of the robust optimization problem and adapts the GAP ball region to construct safe regions for sample screening. By exploiting the strong convexity of the primal objective, the resulting screening rules can safely identify training samples whose uncertainty sets do not affect the optimal classifier, thereby reducing the effective problem size without compromising optimality.

Numerical experiments show that the proposed screening strategy significantly accelerates R-SVM training on real datasets and across various parameter settings while maintaining classification performance. A substantial portion of training samples can be safely eliminated in advance, leading to marked reductions in computational cost.

For future work, we plan to extend this safe screening framework to other robust learning models, including regression tasks such as LASSO, and to investigate whether advanced constructions like the RYU ball can be adapted to robust SVM formulations. We also aim to explore dynamic and hybrid screening strategies to further improve efficiency while maintaining theoretical guarantees.




\bibliographystyle{plainnat}
\bibliography{head-refs}

@article{cortes1995svm,
  title={Support-vector networks},
  author={Cortes, Corinna and Vapnik, Vladimir},
  journal={Machine Learning},
  volume={20},
  number={3},
  pages={273--297},
  year={1995},
  publisher={Springer}
}

@inproceedings{boser1992training,
author = {Boser, Bernhard E. and Guyon, Isabelle M. and Vapnik, Vladimir N.},
title = {A training algorithm for optimal margin classifiers},
year = {1992},
isbn = {089791497X},
publisher = {Association for Computing Machinery},
address = {New York, NY, USA},
url = {https://doi.org/10.1145/130385.130401},
doi = {10.1145/130385.130401},
booktitle = {Proceedings of the Fifth Annual Workshop on Computational Learning Theory},
pages = {144–152},
numpages = {9},
location = {Pittsburgh, Pennsylvania, USA},
series = {COLT '92}
}

@book{vapnik1998statistical,
  title={Statistical Learning Theory},
  author={Vapnik, Vladimir N.},
  year={1998},
  publisher={Wiley},
  address={New York, NY},
  isbn={9780471030034}
}

@InProceedings{ogawa2013safe,
  title = 	 {Safe Screening of Non-Support Vectors in Pathwise SVM Computation},
  author = 	 {Ogawa, Kohei and Suzuki, Yoshiki and Takeuchi, Ichiro},
  booktitle = 	 {Proceedings of the 30th International Conference on Machine Learning},
  pages = 	 {1382--1390},
  year = 	 {2013},
  editor = 	 {Dasgupta, Sanjoy and McAllester, David},
  volume = 	 {28},
  series = 	 {Proceedings of Machine Learning Research},
  address = 	 {Atlanta, Georgia, USA},
  month = 	 {17--19 Jun},
  publisher =    {PMLR},
  url = 	 {https://proceedings.mlr.press/v28/ogawa13b.html},
}

@inproceedings{zhao2014safe,
    author = {Zhao, Zheng and Liu, Jun and Cox, James},
    title = {Safe and efficient screening for sparse support vector machine},
    year = {2014},
    isbn = {9781450329569},
    publisher = {Association for Computing Machinery},
    address = {New York, NY, USA},
    url = {https://doi.org/10.1145/2623330.2623686},
    doi = {10.1145/2623330.2623686},
    booktitle = {Proceedings of the 20th ACM SIGKDD International Conference on Knowledge Discovery and Data Mining},
    pages = {542–551},
    numpages = {10},
    location = {New York, New York, USA},
    series = {KDD '14}
}

@inproceedings{ndiaye2015gap,
 author = {Ndiaye, Eugene and Fercoq, Olivier and Gramfort, Alexandre and Salmon, Joseph},
 booktitle = {Advances in Neural Information Processing Systems},
 editor = {C. Cortes and N. Lawrence and D. Lee and M. Sugiyama and R. Garnett},
 pages = {},
 publisher = {Curran Associates, Inc.},
 title = {GAP Safe screening rules for sparse multi-task and multi-class models},
 url = {https://proceedings.neurips.cc/paper_files/paper/2015/file/69421f032498c97020180038fddb8e24-Paper.pdf},
 volume = {28},
 year = {2015}
}

@inproceedings{ndiaye2016gap,
 author = {Ndiaye, Eugene and Fercoq, Olivier and Gramfort, Alexandre and Salmon, Joseph},
 booktitle = {Advances in Neural Information Processing Systems},
 editor = {D. Lee and M. Sugiyama and U. Luxburg and I. Guyon and R. Garnett},
 pages = {},
 publisher = {Curran Associates, Inc.},
 title = {GAP Safe Screening Rules for Sparse-Group Lasso},
 url = {https://proceedings.neurips.cc/paper_files/paper/2016/file/555d6702c950ecb729a966504af0a635-Paper.pdf},
 volume = {29},
 year = {2016}
}

@article{ndiaye2017gap,
    author = {Ndiaye, Eugene and Fercoq, Olivier and Gramfort, Alexandre and Salmon, Joseph},
    title = {Gap safe screening rules for sparsity enforcing penalties},
    year = {2017},
    issue_date = {January 2017},
    publisher = {JMLR.org},
    volume = {18},
    number = {1},
    issn = {1532-4435},
    journal = {J. Mach. Learn. Res.},
    month = jan,
    pages = {4671–4703},
    numpages = {33}
}

@misc{ogawa2014sample,
  title={Safe Sample Screening for Support Vector Machines}, 
  author={Kohei Ogawa and Yoshiki Suzuki and Shinya Suzumura and Ichiro Takeuchi},
  year={2014},
  eprint={1401.6740},
  archivePrefix={arXiv},
  primaryClass={stat.ML},
  url={https://arxiv.org/abs/1401.6740}, 
}

@article{elghaoui2010,
  title={Safe Feature Elimination in Sparse Supervised Learning},
  author={El Ghaoui, Laurent and Viallon, Vivian and Rabbani, Tarek},
  journal={Pacific Journal of Optimization},
  volume={6},
  number={3},
  pages={667--698},
  year={2010},
  publisher={Yokohama Publishers}
}

@InProceedings{shibagaki2016simultaneous,
  title = 	 {Simultaneous Safe Screening of Features and Samples in Doubly Sparse Modeling},
  author = 	 {Shibagaki, Atsushi and Karasuyama, Masayuki and Hatano, Kohei and Takeuchi, Ichiro},
  booktitle = 	 {Proceedings of The 33rd International Conference on Machine Learning},
  pages = 	 {1577--1586},
  year = 	 {2016},
  editor = 	 {Balcan, Maria Florina and Weinberger, Kilian Q.},
  volume = 	 {48},
  series = 	 {Proceedings of Machine Learning Research},
  address = 	 {New York, New York, USA},
  month = 	 {20--22 Jun},
  publisher =    {PMLR},
  url = 	 {https://proceedings.mlr.press/v48/shibagaki16.html},
}

@article{tran2025one,
  title={One to beat them all:“RYU”--a unifying framework for the construction of safe balls},
  author={Tran, Thu-Le and Elvira, Cl{\'e}ment and Dang, Hong-Phuong and Herzet, C{\'e}dric},
  journal={Open Journal of Mathematical Optimization},
  volume={6},
  pages={1--16},
  year={2025}
}

@InProceedings{fercoq2015mind,
  title = 	 {Mind the duality gap: safer rules for the Lasso},
  author = 	 {Fercoq, Olivier and Gramfort, Alexandre and Salmon, Joseph},
  booktitle = 	 {Proceedings of the 32nd International Conference on Machine Learning},
  pages = 	 {333--342},
  year = 	 {2015},
  editor = 	 {Bach, Francis and Blei, David},
  volume = 	 {37},
  series = 	 {Proceedings of Machine Learning Research},
  address = 	 {Lille, France},
  month = 	 {07--09 Jul},
  publisher =    {PMLR},
  url = 	 {https://proceedings.mlr.press/v37/fercoq15.html},
}

@article{wang2013lasso,
    author = {Wang, Jie and Wonka, Peter and Ye, Jieping},
    title = {Lasso screening rules via dual polytope projection},
    year = {2015},
    issue_date = {January 2015},
    publisher = {JMLR.org},
    volume = {16},
    number = {1},
    issn = {1532-4435},
    journal = {J. Mach. Learn. Res.},
    month = jan,
    pages = {1063–1101},
    numpages = {39}
}

@inproceedings{bonnefoy2014dynamic,
  author={Bonnefoy, Antoine and Emiya, Valentin and Ralaivola, Liva and Gribonval, Rémi},
  booktitle={2014 22nd European Signal Processing Conference (EUSIPCO)}, 
  title={A dynamic screening principle for the Lasso}, 
  year={2014},
  volume={},
  number={},
  pages={6-10},
  doi={}
}

@inproceedings{herzet2022region,
  title={Region-free Safe Screening Tests for $\ell_1$-penalized Convex Problems},
  author={Herzet, C{\'e}dric and Elvira, Cl{\'e}ment and Dang, Hong-Phuong},
  booktitle={2022 30th european signal processing conference (eusipco)},
  pages={2061--2065},
  year={2022},
  organization={IEEE}
}

@inproceedings{dai2012ellipsoid,
  title={An ellipsoid based, two-stage screening test for BPDN},
  author={Dai, Liang and Pelckmans, Kristiaan},
  booktitle={2012 proceedings of the 20th European signal processing conference (EUSIPCO)},
  pages={654--658},
  year={2012},
  organization={IEEE}
}

@inproceedings{mialon2020screening,
  title={Screening data points in empirical risk minimization via ellipsoidal regions and safe loss functions},
  author={Mialon, Gr{\'e}goire and Mairal, Julien and d’Aspremont, Alexandre},
  booktitle={International Conference on Artificial Intelligence and Statistics},
  pages={3610--3620},
  year={2020},
  organization={PMLR}
}

@inproceedings{tran2022beyond,
  title={Beyond GAP screening for Lasso by exploiting new dual cutting half-spaces},
  author={Tran, Thu-Le and Elvira, Cl{\'e}ment and Dang, Hong-Phuong and Herzet, C{\'e}dric},
  booktitle={2022 30th European Signal Processing Conference (EUSIPCO)},
  pages={2056--2060},
  year={2022},
  organization={IEEE}
}

@misc{zhai2019safe,
      title={Safe Sample Screening for Robust Support Vector Machine}, 
      author={Zhou Zhai and Bin Gu and Xiang Li and Heng Huang},
      year={2019},
      eprint={1912.11217},
      archivePrefix={arXiv},
      primaryClass={cs.LG},
      url={https://arxiv.org/abs/1912.11217}, 
}

@article{xu2009robust,
    author = {Xu, Huan and Caramanis, Constantine and Mannor, Shie},
    title = {Robustness and Regularization of Support Vector Machines},
    year = {2009},
    issue_date = {12/1/2009},
    publisher = {JMLR.org},
    volume = {10},
    issn = {1532-4435},
    journal = {J. Mach. Learn. Res.},
    month = dec,
    pages = {1485–1510},
    numpages = {26}
}

@article{yang2018safe,
  title={A safe accelerative approach for pinball support vector machine classifier},
  author={Yang, Zhiji and Xu, Yitian},
  journal={Knowledge-Based Systems},
  volume={147},
  pages={12--24},
  year={2018},
  publisher={Elsevier}
}

@article{pan2017safe,
  title={Safe screening rules for accelerating twin support vector machine classification},
  author={Pan, Xianli and Yang, Zhiji and Xu, Yitian and Wang, Laisheng},
  journal={IEEE transactions on neural networks and learning systems},
  volume={29},
  number={5},
  pages={1876--1887},
  year={2017},
  publisher={IEEE}
}

@unpublished{tran2026new,
  TITLE = {{New Safe Screening Rule for Fast Optimal Transport on Tree, Cycle and Cactus Graphs}},
  AUTHOR = {Tran, Thu-Le and Ngo, Trang and Nguyen, Thanh-Tung and Nguyen, Duy and Nguyen, Ngan and Nguyen, Kien Trung},
  URL = {https://hal.science/hal-05511202},
  NOTE = {working paper or preprint},
  YEAR = {2026},
  MONTH = Feb,
  DOI = {10.13140/RG.2.2.17465.84326},
  HAL_ID = {hal-05511202},
  HAL_VERSION = {v1},
}

@inproceedings{su2024safe,
  title={Safe screening for L2-penalized unbalanced optimal transport problem},
  author={Su, Xun and Fang, Zhongxi and Kasai, Hiroyuki},
  booktitle={2024 International Joint Conference on Neural Networks (IJCNN)},
  pages={1--8},
  year={2024},
  organization={IEEE}
}

@inproceedings{guyard2022screen,
  title={Screen \& relax: accelerating the resolution of Elastic-Net by safe identification of the solution support},
  author={Guyard, Th{\'e}o and Herzet, C{\'e}dric and Elvira, Cl{\'e}ment},
  booktitle={Icassp 2022-2022 ieee international conference on acoustics, speech and signal processing (icassp)},
  pages={5443--5447},
  year={2022},
  organization={IEEE}
}

\end{document}